\documentclass[a4paper]{article}
\usepackage[margin=1.2in]{geometry}

\usepackage[T1]{fontenc}

\usepackage{amsmath, amsthm, amssymb, esint, mathtools, mathabx, tikz-cd}
\usepackage{graphicx}
\usepackage[colorinlistoftodos]{todonotes}
\usepackage[colorlinks=true, allcolors=blue]{hyperref}
\usepackage{caption}

\theoremstyle{thmstyletwo}%
\newtheorem{theorem}{Theorem}[section]
\newtheorem{proposition}[theorem]{Proposition}%

\newtheorem{remark}{Remark}
\newtheorem{lemma}[theorem]{Lemma}%

\newtheorem{definition}{Definition}

\newtheorem{assumption}{Assumption}

\def\XXint#1#2#3{{\setbox0=\hbox{$#1{#2#3}{\int}$ }
\vcenter{\hbox{$#2#3$ }}\kern-.6\wd0}}

\numberwithin{equation}{section}

\begin{document}

\author{A. Martina Neuman\thanks{Department of Computational Mathematics Science and Engineering, Michigan State University, 428 S Shaw Lane, East Lansing, 48824, USA; E-mail: \texttt{neumana6@msu.edu}.}
}

\date{}

\vspace{-0.25in}

\title{Graph Laplacians on Shared Nearest Neighbor graphs and graph Laplacians on $k$-Nearest Neighbor graphs having the same limit}

\maketitle

\begin{abstract}
A Shared Nearest Neighbor (SNN) graph is a type of graph construction using shared nearest neighbor information, which is a secondary similarity measure based on the rankings induced by a primary $k$-nearest neighbor ($k$-NN) measure. SNN measures have been touted as being less prone to the curse of dimensionality than conventional distance measures. Thus, methods using SNN graphs have been widely used in applications, particularly in clustering high-dimensional data sets and finding outliers subspaces of high dimensional data. Despite this, the theoretical study of SNN graphs and graph Laplacians remains unexplored. In this pioneering work, we make the first contribution in this direction. We show that large scale asymptotics of an SNN graph Laplacian reach a consistent continuum limit; this limit is the same as that of a $k$-NN graph Laplacian. Moreover, we show that the pointwise convergence rate of the graph Laplacian is linear with respect to $(k/n)^{1/m}$ with high probability.
\end{abstract}



\section{Introduction}

\begin{figure}
    \centering
    \includegraphics[width=1\textwidth]{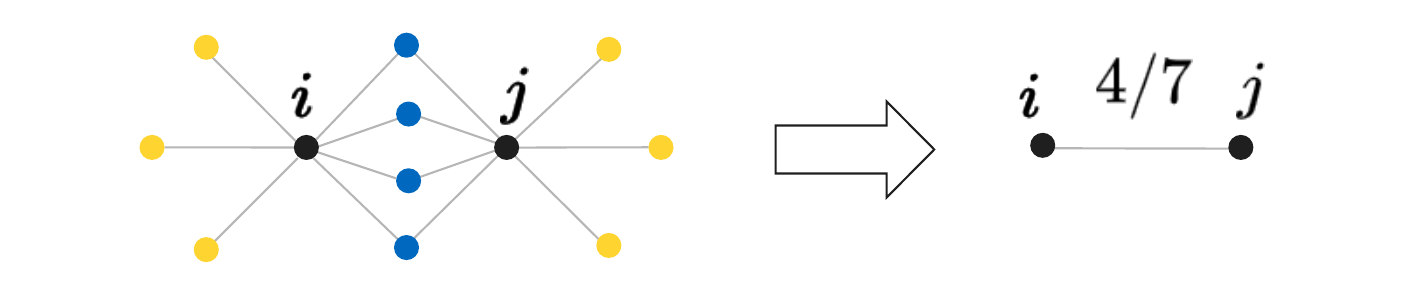}
    \caption{An SNN graph constructed on top of a $7$-NN graph: here, nodes $i$ and $j$ have $4$ $7$-nearest neighbors; hence, they are SNN neighbors of each other and assigned a ({\it simcos}) edge weight of $4/7$.}
    \label{fig:SNN}
\end{figure}

Graph Laplacians are undoubtedly a ubiquitous tool in machine learning. In machine learning, when a data set $X=\{x_1,\cdots,x_{n}\}\subset\mathbb{R}^{d}$ is sampled out of a data generating measure $\mu$ supported on a Riemannian submanifold $\mathcal{M}$, their neighborhood graphs are treated as discrete approximations of $\mathcal{M}$, and the spectrum of the resulted graph Laplacians are used to extract its intrinsic structural information. For example, if the manifold is of a low dimension $m$, then this dimension can be detected by the first $m$ eigenvectors corresponding to the $m$ largest eigenvalues of the graph Laplacians. Even though in practice, the underlying manifold (and the associated measure) is technically a priori unknown or unseen, this so-called {\it manifold assumption} \cite{chapelle2006semi} is fairly common, since strong dependencies are often exhibited between the individual feature vectors $x_{i}$. It has become the basis for many dimension reduction methods using spectra of graph Laplacians \cite{belkin2003laplacian, coifman2006diffusion, coifman2005geometric}. \\
Methods using graph Laplacians are remarkably successful in semi-supervised learning \cite{zhu2002learning, zhu2003semi, belkin2002semi}, due to that graph Laplacians are generators of the diffusion process on graphs, therefore suitable for studying label propagation, and in spectral clustering \cite{von2007tutorial, ng2001spectral}, due to the special properties possessed by their eigenvectors \cite{mohar1997some}. Similarly, (weighted) Laplace-Beltrami operators are the generators of the diffusion process on manifolds, and their spectra capture important geometric properties of the manifolds \cite{chavel1984eigenvalues}. Since point clouds arising as discretizations of Riemannian manifolds are distinguishable from arbitrary ones, it's reasonable to presume that large sample size asymptotics of their graph Laplacians mimic the behavior of the (weighted) manifold Laplace-Beltrami operators. It turns out that a passage from the discrete operator to the continuum one depends on (1) the graph and (2) the Laplacian constructions, as well as (3) the rate at which the graph connectivity parameters go to zero. We discuss this matter in more detail below, starting with an introduction to two well-known graph generations and the Shared Nearest Neighbor graphs. \\

Roughly speaking, a graph on $X$ is built on a similarity measure. In many graphs, this similarity measure is both {\it primary} and distance-based. For example, in an $\epsilon$-graph, two points $x_{i}, x_{j}\in X$ are connected if their Euclidean distance in $\mathbb{R}^{d}$ is at most some chosen $\epsilon>0$. In another, closely related, {\it directed} $k$-Nearest Neighbor graph ($k$-NN), a directed edge is drawn from $x_{i}$ to $x_{j}$, if $x_{j}$ is among the $k$-nearest neighbors of $x_{i}$, for some $k\in\mathbb{N}$ of choice. Both graphs are therefore {\it proximity} graphs \cite{toth2017handbook}, with $k$-NN graph additionally connecting vertices by ranking their distances. However, it's been known that similarity measures based on distances are sensitive to variations within a data distribution, or the ambient dimension $d$ (in fact, questions were raised as to whether the concept of the nearest neighbor is meaningful in high dimensions \cite{beyer1999nearest}). The need for a similarity measure that is better at handling high dimensional data led to the invention of {\it secondary} similarity measure. A special example of graphs constructed from this type of similarity measure is the Shared Nearest Neighbor (SNN) graph, the subject of our investigation. Typically, in an SNN graph, once the primary similarity of $k$-nearest neighbors is used - where $k$-nearest neighbors for each point $x_{i}$ are already determined - a secondary similarity is applied, by ranking affinity induced by the primary one: $x_{i}, x_{j}$ are connected if $x_{i}, x_{j}$ share a $k$-nearest neighbor in common. More concretely, let $NN(x)\subset X\setminus\{x\}$ be the set of $k$-nearest neighbors of $x$ and $card(NN(x))$ denote its cardinality. Then $x_{i}, x_{j}$ are SNN neighbors if 
\begin{equation*} 
    card(NN(x_{i})\cap NN(x_{j}))>0
\end{equation*}
(see Figure \ref{fig:SNN} for an illustration). An undirected edge is drawn between them with an edge weight of either the intersection size $card(NN(x_{i})\cap NN(x_{j}))$ or the {\it cosine measure} \cite{houle2010can}
\begin{equation} \label{def:cosineweight}
    \text{simcos }(x_{i},x_{j}):=\frac{card(NN(x_{i})\cap NN(x_{j}))}{k}, 
\end{equation}
aptly named since it's equivalent to the cosine of the angle between the zero-one set membership vectors of $NN(x_{i})$ and $NN(x_{j})$. This measure \eqref{def:cosineweight} was often used as a local density for clustering \cite{ertoz2003finding, houle2003navigating}. It's been reported, and empirically confirmed in \cite{houle2010can}, that SNN measures are stable and less prone to the curse of high dimensions than conventional distance measures. As such, they have found use in clustering algorithms for large or high dimensional data sets \cite{ertoz2003finding, guha1998cure, jarvis1973clustering, houle2003navigating, houle2008relevant} as well as in finding outliers in high dimensions \cite{kriegel2009outlier}. Despite their popularity with the computing community \cite{kumari2016scalable, xu2015identification, ertoz2004finding, faustino2014kd, antunes2014fast}, the theoretical understanding of SNN graphs and their Laplacians remain lacking, to the author's knowledge.\\

There are three main types of graph Laplacians studied so far in machine learning; they are, normalized, unnormalized and random walk Laplacians. Precise definitions will be given in \ref{sec:resstate}. A plethora of work has been devoted to the pointwise consistency of graph Laplacians on $\epsilon$-graphs \cite{singer2006graph, hein2007graph, ting2011analysis, belkin2005towards, gine2006empirical, calder2022improved}, but to a lesser extent, of $k$-NN graph Laplacians \cite{calder2022improved, cheng2022convergence}. The key idea in all these pointwise convergence results is that the graph connectivity parameter must be kept small yet relatively fixed with respect to the sample size $n$ as $n\to\infty$. Notably, it was shown in \cite{hein2007graph} for $\epsilon$-graphs, where this parameter is precisely $\epsilon$, that if $\epsilon\to 0$ and $n\epsilon^{m+2}/\log n\to\infty$, then
\begin{equation} \label{hein}
    \mathcal{L}^{\epsilon}f(x)\to Cp^{-1}(x)div(p^2\nabla f)(x)
\end{equation}
almost surely for a non-boundary point $x\in\mathcal{M}$, where $\mathcal{L}^{\epsilon}$ denotes the unnormalized $\epsilon$-graph Laplacian and $p$ the nonvanishing density of $\mu$. The optimal rate in \eqref{hein} is $\epsilon=\epsilon(n)=O((\log n/n)^{1/(m+4)})$. In \cite{calder2022improved}, \eqref{hein} was recovered for a compact, boundariless manifold $\mathcal{M}$. A consistency result for {\it undirected} $k$-NN graph was also established in the same paper, where the authors showed that, with high probability,
\begin{equation} \label{calder}
    \mathcal{L}^{k}f(x_{i})\to Cp^{-1}(x)div(p^{1-2/m}\nabla f)(x_{i})
\end{equation}
uniformly for every $x_{i}\in X$ and linearly in terms of $(k/n)^{1/m}$, whenever 
\begin{equation} \label{logscale}
    C(\log n)^{\frac{m}{m+4}}n^{\frac{4}{m+4}}\ll k\ll n.
\end{equation}
Here, $\mathcal{L}^{k}$ denotes the unnormalized $k$-NN graph Laplacian, and $(k/n)^{1/m}$ plays the role of graph connectivity parameter. \\

The inspiration for this work starts with the paper \cite{calder2022improved}. In a similar spirit with \eqref{calder}, we seek to uncover the effective limit operator of the unnormalized SNN graph Laplacian as $(k/n)^{1/m}\to 0$. We reveal that this limit is the same as in \eqref{calder},
\begin{equation*}
    -\frac{1}{2p}\nabla\cdot (p^{1-2/m}\nabla) = -\frac{1}{2p}div(p^{1-2/m}\nabla)=:\Delta^{snn}
\end{equation*}
This means, although an SNN graph is built on a $k$-NN graph, their Laplacians converge to the same operator. It also means that manifold spectral information is saturated at the primary level with the use of $k$-NN graph. Furthermore, we show that when \eqref{logscale} is satisfied, then with high probability,  
\begin{equation} \label{neuman}
    |\mathcal{L}^{snn}f(x_{i})-\frac{1}{m+2}\Delta^{snn}f(x_{i})|=O((k/n)^{1/m}),
\end{equation}
where $\mathcal{L}^{snn}$ denotes the unnormalized SNN graph Laplacian. To our knowledge, we're the first to establish the pointwise consistency result for SNN graph Laplacians. It is expected to serve as the first installment toward studying consistency of graph-based algorithms on SNN graphs. Particularly, we plan to investigate the convergence of SNN-graph-based spectral clustering, which we expect the nonasymptotic and quantitative nature of \eqref{neuman} would be greatly beneficial.\\

\noindent {\it Outline.} The paper is organized as follows. In Section \ref{sec:setup}, we give all the basic set-up and precise constructions of SNN graphs and graph Laplacians. In the first half of Section \ref{sec:main}, we state our assumptions, our main result as well as its ramifications. In the second half, we give an outline of the proof and recall necessary geometry and concentration results. In Section \ref{sec:mainproof}, we present our main proof, with some tedious steps abstracted away in the Appendix \ref{appx}.

\section{Set-up} \label{sec:setup}

\subsection{Preliminary} 

{\it Basic manifold set-up.} Let $\mathcal{M}$ be a compact, connected, boundariless, orientable, smooth $m$-dimensional manifold ($m\geq 2$) embedded in $\mathbb{R}^{d}$. Some essential constants intrinsic to $\mathcal{M}$ are: an upper bound $K$ on the absolute values of the sectional curvatures, the reach $R$ of $\mathcal{M}$ and a lower bound $i_0$ on the injectivity radius of $\mathcal{M}$. We let $\mathcal{M}$ inherit the Riemannian structure induced by the ambient space $\mathbb{R}^{d}$. We write $d\mathcal{V} = dVol_{\mathcal{M}}$ to denote the volume form on $\mathcal{M}$ with respect to the induced metric tensor, and $d(x,y)$ to denote the geodesic distance between $x,y\in\mathcal{M}$. Furthermore, let $\mu$ be a probability measure supported on $\mathcal{M}$ and $p$ be its density, such that
\begin{equation} \label{densitybds}
    0<p_{min}\leq p(x)\leq p_{max}<\infty
\end{equation}
for every $x\in\mathcal{M}$. We assume $p\in C^2(\mathcal{M})$; i.e., $p$ (expressed in normal coordinates) has continuous second partial derivatives (see also Remark \ref{rem:tilde} below). These regularity assumptions on $p$ are fairly common in theoretical works on graph based learning, as they allow for tangible connections between learning algorithms and PDE theory to be established. \\

\noindent {\it Basic analytic set-up.} We define $L^2(\mu)$ to be the space of $L^2$-functions on $\mathcal{M}$ with respect to $\mu$, endowed with the inner product
\begin{equation} \label{def:L^2inn}
    \langle f, g\rangle_{\mu} := \int_{\mathcal{M}} f(x)g(x)\,d\mu(x)=\int_{\mathcal{M}} f(x)g(x)p(x)\,d\mathcal{V}(x) \quad f, g\in L^2(\mu). 
\end{equation}
We say $f\in L^2(\mu)$ if $\langle f, f\rangle_{\mu}=:\|f\|^2_{\mu}<\infty$.\\

When $X=\{x_1,\cdots, x_{n}\}$ is a set of i.i.d. samples from $\mu$, we let $\mu_{n}$ be the usual empirical measure
\begin{equation} \label{def:empm}
    \mu_{n}:=\frac{1}{n}\sum_{i=1}^{n}\delta_{x_{i}}.
\end{equation}
Similarly as above, we define $L^2(\mu_{n})$ to be the space of functions on $X$ with the inner product
\begin{equation*} 
    \langle u, v\rangle_{\mu_{n}} := \frac{1}{n}\sum_{i=1}^{n} u(x_{i})v(x_{i}) \quad u, v\in L^2(\mu_{n}),
\end{equation*}
and $\|u\|^2_{\mu_{n}}:=\langle u,u\rangle_{\mu_{n}}$.\\

\noindent {\it Basic notation agreement.} We allow abstract, analytic inequality constants $C,c$ to change their values from one line to the next; moreover, they are implicitly dependent on the following intrinsic values of the manifold: $m, K, i_0, R$ as well as $\alpha$, where $\alpha:=Vol_{m}(B(0,1))$, the $m$-dimensional volume of the unit ball. We will often indicate but not fully disclose an analytic constant's dependence on the density $p$. For instance, the following constant will play a role in our analysis to ensure the convergence of the SNN graph Laplacian,
\begin{equation} \label{def:epsM}
    c_{\mathcal{M}}:=\alpha^{1/m} p_{min}^{1/m}\min\{1,i_0, K^{-1/2}, R/2\};
\end{equation}
we can simply write $c_{\mathcal{M}} = C_{p}\min\{1,i_0, K^{-1/2}, R/2\}$. We will also write $A\asymp_{p} B$ to mean $A\leq C_{p}B$ and $B\leq C_{p}A$. In various statements, we choose different expressions of constant parametric dependence; these notations are local and defined where they are used.\\
To distinguish different types of geometric balls, we write $\mathcal{B}(x,r)$ to denote a geodesic ball in $\mathcal{M}$ with center $x$ and radius $r$, $B(x,r)$ to denote a Euclidean ball, and $\overline{B(x,r)}$ its topological closure, either in $\mathbb{R}^{d}$ or $\mathbb{R}^{m}$, depending on context.\\
We abuse the use of the notation $|\cdot|$, which either means an absolute value of a quantity, or a Euclidean vector norm, or Lebesgue measure of a set, depending on context. Finally, we define an (essential) support of a function $f: \mathbb{R}^{m}\supset U\to\mathbb{R}$ to be
\begin{equation*}
    supp(f):= U\setminus\bigg(\bigcup\{B(x,r):x\in\mathbb{R}^{m}, r>0 \text{ s.t. } |\{z\in U:f(z)\not= 0\}\cap B(x,r)|=0\}\bigg).
\end{equation*}

\subsection{Basic graph Laplacian constructions} \label{sec:resstate}

Let $\Gamma=(X,\{e_{ij}\}_{ij})$ denote an undirected, weighted (finite) graph on the set of nodes $X=\{x_1,\cdots,x_{n}\}$, where each edge $e_{ij}$ is given an edge weight $w_{ij}$. Let $W\in\mathbb{R}^{n\times n}$ be a symmetric matrix whose $ij$th entry is $w_{ij}$ ({\it weight matrix}) and $D\in\mathbb{R}^{n\times n}$ be diagonal whose $ii$th entry is $d_{i}:=\frac{1}{n}\sum_{j} w_{ij}$ ({\it degree matrix}).\\
The unnormalized or combinatorial graph Laplacian \cite{hein2007graph, chung1996combinatorial} is defined to be $\mathcal{L}^{(u)} := D-W$, or
\begin{equation} \label{def:ugraphLap}
    \mathcal{L}^{(u)}f(x_{i}) := d_{i}f(x_{i})-\frac{1}{n}\sum_{j}w_{ij}f(x_{j})
\end{equation}
for $f\in L^2(\mu_{n})$. To compare, the normalized and random walk graph Laplacians are defined respectively as follows \cite{hein2007graph},
\begin{equation*}
    \mathcal{L}^{(n)}:=I-D^{-1/2}WD^{-1/2} \quad\text{ and }\quad \mathcal{L}^{(rw)}:=I-D^{-1}W,
\end{equation*}
where $I$ stands for the identity matrix. We note that when $w_{ij}$ takes the form of a kernel function of the distance between $x_{i}, x_{j}$, e.g., $w_{ij}=\upsilon(|x_{i}-x_{j}|)$, for some non-increasing function $\upsilon: [0,\infty)\to [0,1]$, the graph $\Gamma$ is a proximity graph. As we shall see below, an SNN graph is not a proximity graph.  

\subsection{SNN graph and SNN graph Laplacian}

We first introduce a $k$-relation, which we will use to define neighbors on an SNN graph of the data set $X=\{x_1,\cdots,x_{n}\}$.

\begin{definition} \label{def:krelations} We define a relation $\sim_{knn}$ on $X\times X$ by declaring
\begin{equation} \label{def:krelation}
    x_{i}\sim_{knn} x_{j}
\end{equation}
if $x_{j}\in X\setminus\{x_{i}\}$ is among the $k$-nearest neighbors (in Euclidean distance) of $x_{i}$. 
\end{definition}

Note that $\sim_{knn}$ in \eqref{def:krelation} is a {\it directed, anti-symmetric} relation. A {\it symmetric} relation, $x_{i}\sim x_{j}$ if $x_{i}\sim_{knn} x_{j}$ or $x_{j}\sim_{knn} x_{i}$, or a {\it mutual} one, $x_{i}\sim x_{j}$ if $x_{i}\sim_{knn} x_{j}$ and $x_{j}\sim_{knn} x_{i}$, can be defined out of \eqref{def:krelation}, but we will only need it for the following definition of SNN neighbors.

\begin{definition} \label{def:snn}
Let $x_{i}, x_{j}\in X$. We say that $x_{i}, x_{j}$ are SNN neighbors if there exists $x_{l}\in X\setminus\{x_{i},x_{j}\}$ such that
\begin{equation} \label{sneighbors}
    x_{i} \sim_{knn} x_{l} \quad\text{ and }\quad x_{j}\sim_{knn} x_{l}.
\end{equation}
Moreover, when \eqref{sneighbors} happens, we write, $x_{i}\sim_{snn} x_{j}$, and say that $x_{l}$ is a shared $k$-NN neighbor of both $x_{i}, x_{j}$.
\end{definition}

Hence, in an SNN graph $\Gamma$, an edge $e_{ij}$ exists between $x_{i}, x_{j}$ if both nodes have a common neighbor $x_{l}$ in the sense of \eqref{sneighbors}. The defined $\sim_{snn}$ relation is symmetric, and the resulted SNN graph $\Gamma = (X,\{e_{ij}\}_{i,j})$ is undirected. To complete the construction, we need to assign each edge a weight that reflects the number of shared $k$-NN neighbors between the edge nodes. We do this next.


\subsubsection{Unnormalized SNN graph Laplacian}

Since the density $p$ is bounded below (\eqref{densitybds}), with probability one, the requirement \eqref{def:krelation} is equivalent to
\begin{equation} \label{krelationequiv}
    \mu_{n}(\overline{B(x_{i},r)})\leq \frac{k}{n}
\end{equation}
where $r=|x_{i}-x_{j}|$ and $\mu_{n}$ is the empirical measure in \eqref{def:empm}. Therefore, \eqref{krelationequiv} can serve as a quantification of \eqref{def:krelation}. Following this, we define, for every $\epsilon>0$
\begin{equation} \label{def:Ne}
    N_{\epsilon}(x) :=\sum_{i: 0<|x_{i}-x|\leq\epsilon} 1.
\end{equation}
Now $N_{\epsilon}(x)$ captures the number of random samples $x_{i}$ in the punctured Euclidean $\epsilon$-neighborhood of $x$. It's most fitting to take $x\in X$ in \eqref{def:Ne}; however, $x$ can also be a location on the manifold. A known estimate for $N_{\epsilon}(x)$ is as follows.

\begin{lemma} \label{calderlem:N} \cite[Lemma 3.8]{calder2022improved} Let $x\in\mathcal{M}$ and suppose that $\epsilon<c_{\mathcal{M}}$ (see \eqref{def:epsM}). Then for $\epsilon^2\leq\delta\leq 1$,
\begin{equation*} 
    \mathbb{P}(|N_{\epsilon}(x)-\alpha p(x)n\epsilon^{m}|\geq C\delta n\epsilon^{m})\leq 2\exp(-c\delta^2 n\epsilon^{m}). 
\end{equation*}
\end{lemma}

Following \cite{calder2022improved}, we let
\begin{equation} \label{def:epsk}
    \varepsilon_{k}(x) := \min\{\epsilon>0: N_{\epsilon}(x)\geq k\}.
\end{equation}
As before, $x$ can be a point on the manifold. By \eqref{krelationequiv}, $N_{\varepsilon_{k}}(x)=k$, if $x\in X$. A similar statement can be said for general $x\in\mathcal{M}$. To see this, we mention the following result which dictates that $\varepsilon_{k}(x)^{m}\asymp_{p} k/n$ with high probability.

\begin{lemma} \label{calderlem:epsk} \cite[Lemma 3.9]{calder2022improved} 
Let $1\leq k\leq cnc_{\mathcal{M}}^{m}$. Then for $x\in\mathcal{M}$ and $C(k/n)^{2/m}\leq\delta\leq 1$,
\begin{equation*}
    \mathbb{P}(|\alpha p(x)\varepsilon_{k}(x)^{m}-k/n|\geq C\delta k/n)\leq 4\exp(-c\delta^2 k).
\end{equation*}
\end{lemma}

Then simple calculations from Lemmas \ref{calderlem:N}, \ref{calderlem:epsk} show that, for $x\in\mathcal{M}$ and $C(k/n)^{2/m}\leq\delta\leq 1$,
\begin{equation} \label{kconc}
    \mathbb{P}(|N_{\varepsilon_{k}}(x)-k|\geq C\delta k)\leq 6\exp(-c\delta^2k).
\end{equation}
Denote $B_{k}(x):= B(x,\varepsilon_{k}(x))$. Motivated by \eqref{def:epsk}, \eqref{kconc}, we quantify the number of shared $k$-NN neighbors between $x,y\in X$ as
\begin{equation} \label{def:sharednn}
    N(x,y) := \sum_{x_{i}\not= x,y} 1_{\overline{B_{k}(x)}}(x_{i})1_{\overline{B_{k}(y)}}(x_{i}) =  \sum_{x_{i}\not= x,y} \eta\bigg(\frac{|x-x_{i}|}{\varepsilon_{k}(x)}\bigg)\eta\bigg(\frac{|y-x_{i}|}{\varepsilon_{k}(y)}\bigg),
\end{equation}
where $\eta:=1_{[0,1]}(t)$, $t\geq 0$, is the Heaviside step function. Note that, with probability one, the number of shared $k$-nearest neighbors between $x,y\in X$ is $l\leq k$ iff $N(x,y) = l$.\\
We now assign an edge $e_{ij}$ in $\Gamma$ a weight $w_{ij}:=N(x_{i},x_{j})/k$; note that this is the cosine measure mentioned in \eqref{def:cosineweight}. We construct an unnormalized SNN graph Laplacian on $L^2(\mu_{n})$ as follows,
\begin{equation} \label{def:graphLap}
    \mathcal{L}^{snn}u(x_{i}) := \frac{(\alpha n/k)^{1+2/m}}{2^{m+2} n}\sum_{j=1}^{n} w_{ij}(u(x_{i})-u(x_{j}))=\frac{(\alpha n/k)^{1+2/m}}{2^{m+2} n}\sum_{j=1}^{n} \frac{N(x_{i},x_{j})}{k}(u(x_{i})-u(x_{j})).
\end{equation}
A few remarks are in order.

\begin{remark} 
Implicit in the formulation \eqref{def:graphLap} is that $h:=2(k/(\alpha n))^{1/m}$ acts as the graph connectivity parameter. This is suggested by Lemma \ref{calderlem:epsk}, \eqref{def:epsk} and our observation of $N(x,y)$ above, which all conclude that $h$ is an expected SNN neighborhood radius at each datum $x\in X$. \\
When introducing a factor $1/h$, $(\alpha n/k)/2^{m}=1/h^{m}$ becomes a necessary rescaling to reveal a divergent form at the microscopic level, whereas the factor $(\alpha n/k)^{2/m}/2^2=1/h^2$ arises because the Laplacian corresponds to a second derivative. The factor $1/n$ will later play a role in a concentration effect.\\
We will gather information about allowed choice of the parameter $k$ for consistency in \ref{sec:main}; however $k$ should be such that $h\to 0$ sufficiently slow when $n\to\infty$. Then when the number of points in each datum neighborhood encroaches infinity, heuristically, the sum \eqref{def:graphLap} approximates an integral whose normalization approaches $\Delta^{snn}f(x_{i})$. This is the basic principle behind all the graph Laplacian convergence results \cite{hein2007graph, calder2022improved, burago2015graph} and is a well-known principle in the framework of nonparametric regression \cite{gyorfi2002distribution}. 
\end{remark}

\begin{remark} \label{rem:graphcons}
Up to the scaling of $1/h^{m+2}$, the formulation \eqref{def:graphLap} depicts a standard unnormalized graph Laplacian \eqref{def:ugraphLap}. Indeed $h^{m+2}\mathcal{L}^{snn} = D - W$, where 
\begin{equation*}
    d_{i} = \frac{1}{n}\sum_{j}\frac{N(x_{i},x_{j})}{k} \quad\text{ and }\quad w_{ij} = \frac{N(x_{i},x_{j})}{k}.
\end{equation*}
It can be seen from construction that $w_{ij}$ does not only depend on the distance $|x_{i}-x_{j}|$ but also on the locations of $x_{i}, x_{j}$. Therefore, it is not a radial edge weight. This lack of radiality, as we shall see, is a complication in deriving the limiting operator for $\mathcal{L}^{snn}$ as $n\to\infty$.
\end{remark}

We're ready to state our result.

\section{Main result} \label{sec:main}

Given the set-up in \ref{sec:setup}, we show that, with high probability, $\mathcal{L}^{snn}$ in \eqref{def:graphLap} converges pointwise, with a linear rate, to (a multiple of) the following weighted Laplace-Beltrami operator on $\mathcal{M}$:
\begin{equation} \label{def:manLap}
    \Delta^{snn} f(x) := -\frac{1}{2p(x)}div(p^{1-2/m}\nabla f)(x) \quad \text{ for }\quad f\in C^2(\mathcal{M}).
\end{equation}
The notation $div$ stands for the divergence operator on $\mathcal{M}$, and $\nabla$ for the gradient. A precise statement is given in Theorem \ref{thm:consistency} below.

\begin{assumption} \label{assumption} Let $m\geq 2$ and denote $c'_{p,\mathcal{M}}:=\max_{x\in\mathcal{M}} |\nabla p(x)|$. We assume the following for the remainder of this paper:
\begin{enumerate}
    \item $3(k/n)^{1/m}<c_{\mathcal{M}}$, where $c_{\mathcal{M}}$ is the intrinsic constant defined in \eqref{def:epsM},
    \item $(k/n)^{1/m}< C\alpha^{1/m}p_{min}^{1+1/m} (c'_{p,\mathcal{M}})^{-1}$,
    \item $(\log n)^{\frac{m}{m+4}} n^{\frac{4}{m+4}}\ll k\ll n$.
\end{enumerate}
\end{assumption}

Let $\sigma :=\int_{B(0,1)} u_1|^2\eta(|u|)\,du$, the {\it surface tension} of $\eta$, where $B(0,1)$ here is the unit ball in $\mathbb{R}^{m}$ and $u_1$ the first coordinate of $u$. It's known that $\sigma = \frac{\alpha}{m+2}$ \cite{garcia2020error}. For $l=0,1,2,3$ and $f\in C^{l}(\mathcal{M})$, we define
\begin{equation*}
    \|f\|_{C^{l}(\mathcal{M})} := \sum_{j=0}^{l} \|\nabla^{(j)}f\|_{L^{\infty}(\mathcal{M})}.
\end{equation*}

\begin{theorem} \label{thm:consistency} Let $f\in C^3(\mathcal{M})$. Then for $ (k/n)^{1/m}\leq t\leq (k/n)^{-1/m}$, we have
\begin{equation} \label{thmstatement}
    \mathbb{P}\bigg(\max_{1\leq i\leq n}|\mathcal{L}^{snn}f(x_{i})-\frac{1}{m+2}\Delta^{snn}f(x_{i})|\geq Ct\bigg)\leq C_{p}n\exp(-c_{p}t^2 (k/n)^{2/m}k)
\end{equation}
where $C=C(p,\|f\|_{C^3(\mathcal{M})})$ in \eqref{thmstatement} denotes a constant depending on $p$ and $\|f\|_{C^3(\mathcal{M})}$. 
\end{theorem}

\begin{remark}
Note that the second point of Assumption \ref{assumption} is purely technical. One can combine the first and the second points; we separate them since the second involves the knowledge of the global maximum of $|\nabla\log p|$. See also Remark \ref{rem:secondpt} below. The third point in Assumption \ref{assumption} ensures that the probability bound in \eqref{thmstatement} is strictly less than one, while the role of the first will be clear in \ref{sec:Riegeo}.   
\end{remark}

\subsection{Ramifications of Theorem \ref{thm:consistency}} \label{sec:rami}

Note from construction that SNN graphs adjust to data density in a way that is different from, say, the $\epsilon$- and $k$-NN- graphs \cite{garcia2019variational, calder2022improved, hein2007graph, garcia2020error}. Yet Theorem \ref{thm:consistency} shows that the SNN graph Laplacian reaches the same continuum operator as the $k$-NN graph Laplacian, which is \cite{calder2022improved}
\begin{equation} \label{snnknneq}
    \Delta^{snn} = -\frac{1}{2p}div(p^{1-2/m}\nabla) := \Delta^{knn}.
\end{equation}
Hence an application of SNN graph doesn't produce any more manifold spectral information than what could be gained from a $k$-NN graph. \\
Recall from \cite{hein2007graph} the following definition of the $s$-th weighted Laplace-Beltrami operator
\begin{equation} \label{def:weightLB}
    \Delta_{s}:=-\frac{1}{2p^{s}}div(p^{s}\nabla) = -\frac{1}{2}\bigg(-\Delta + s\nabla\log p\cdot\nabla\bigg).
\end{equation}
Here $s\in\mathbb{R}$, and $\Delta:=\nabla\cdot\nabla$ is the unweighted Laplace-Beltrami on $\mathcal{M}$. The presence of $p^{s}$ in \eqref{def:weightLB} reveals how the data-dependent nature of the edge weights in a given graph construction influences the limiting differential operator. Since a Laplacian operator generates a diffusion process, the term $s\nabla\log p\cdot\nabla$ can be seen as inducing an anisotropic term, which directs the diffusion toward ($s>0$) or away from ($s<0$) increasing density. An interesting finding in \cite{hein2007graph} is that, for a given graph based on primary similarity, different types of graph Laplacians converge to different scaled versions of $\Delta_{s}$. In particular, for an unnormalized graph Laplacian, this limit is 
\begin{equation} \label{rule}
    p^{1-2\lambda}\Delta_{s} \quad\text{ where }\quad s=:2(1-\lambda).
\end{equation}
Since the limiting operator for the unnormalized $\epsilon$-graph Laplacian is $\Delta^{\epsilon}:=-\frac{1}{2p}div(p^2\nabla)$ \cite{hein2007graph, garcia2020error, calder2022improved}, we see that $1-2\lambda=1$ and $s=2$ in this case. Although the $k$-NN graph as well as the SNN graph was not among the ones considered in \cite{hein2007graph}, referring \eqref{rule} back to \eqref{snnknneq}, we find $1-2\lambda=-2/m$ and $s=1-2/m$ for both cases. This suggests that \eqref{rule} might hold for more than just non-ranking, primary proximity based graphs. \\
The fact that $s=1-2/m\geq 0$ in \eqref{snnknneq} for all $m\geq 2$ is desirable in clustering and classification (see \ref{sec:spectrum} below) where one wants the diffusion process mainly along the regions of the same density level. However, due to the factor $p^{1-2\lambda}=p^{-2/m}$ in the limit, the unnormalized SNN (and $k$-NN) graph Laplacian is predicted to be unfit for applications such as label propagation, since the propagation will be slow in regions of high densities \cite{hein2007graph}. 

\subsubsection{Spectral information} \label{sec:spectrum}

Since we intend to use this exposition as a starting point for future work on spectral convergence of SNN graph Laplacians, we briefly discuss the spectra of $\mathcal{L}^{snn}$ and $\Delta^{snn}$ here.\\
In what follows, $H^1(\mu)$ denotes the Sobolev space of functions $f$ with a weak first derivative $\nabla f$, and $H^2(\mu)$, with a weak second derivative $\nabla^2 f$, all in $L^2(\mu)$ \cite{evans2010partial}. Then $\|f\|_{H^1(\mu)}^2:=\|f\|_{L^2(\mu)}^2 + \|\nabla f\|_{L^2(\mu)}^2$. The definition for $\|f\|_{H^2(\mu)}^2$ is similar. Induced by $\Delta^{snn}$ is a manifold Dirichlet energy functional on $\mathcal{M}$
\begin{equation} \label{def:manE}
    \mathcal{E}^{\mathcal{M}}(f) := \int_{\mathcal{M}} |\nabla f(x)|^2p(x)^{1-2/m}\,d\mathcal{V}(x) \quad \text{ for }\quad f\in H^1(\mu)
\end{equation}
and an associated Dirichlet bilinear form,
\begin{equation*} 
    B(f,g) := \int_{\mathcal{M}} \nabla f(x)\cdot\nabla g(x)p(x)^{1-2/m}\,d\mathcal{V}(x), \quad \text{ for }\quad f,g\in H^1(\mu).
\end{equation*}
Since $p$ is bounded \eqref{densitybds}, $B(f,g)$ is nonvanishing and continuous: $|B(f,g)|\leq C\|f\|_{H^1(\mu)}\|g\|_{H^1(\mu)}$. The relationship between $\Delta^{snn}$ and $\mathcal{E}^{\mathcal{M}}$ can be seen as follows. For $f\in H^2(\mu), g\in H^1(\mu)$, by definition \eqref{def:L^2inn} and integration by parts
\begin{equation} \label{endelta}
    2\langle \Delta^{snn} f,g\rangle_{\mu} = 2\int_{\mathcal{M}} \Delta^{snn} f(x)g(x)p(x)\,d\mathcal{V}(x) = \int_{\mathcal{M}} \nabla f(x)\cdot\nabla g(x)p(x)^{1-2/m}\,d\mathcal{V}(x) = B(f,g).
\end{equation}
Hence, $2\langle \Delta^{snn} f,f\rangle_{\mu}=B(f,f)=\mathcal{E}^{\mathcal{M}}(f)\geq 0$, and so $\Delta^{snn}$ is a positive semidefinite operator on $H^2(\mu)$; as such, it has a pure point spectrum whose eigenvalues, including (finite) multiplicity, can be listed in an increasing order
\begin{equation*}
    0=\lambda^{\mathcal{M}}_0<\lambda^{\mathcal{M}}_1\leq\lambda^{\mathcal{M}}_2\leq\cdots\leq\lambda^{\mathcal{M}}_{l}\leq\cdots .
\end{equation*}
By \eqref{def:manE} and the Rayleigh-Ritz variational principle 
\begin{equation} \label{lambda1}
    \lambda^{\mathcal{M}}_1 = \inf_{f\in C^{\infty}(\mathcal{M})} \bigg\{\frac{\int_{\mathcal{M}} |\nabla f(x)|^2 p(x)^{1+2/m}\,d\mathcal{V}(x)}{\int_{\mathcal{M}} |f(x)|^2 p(x)^{1+2/m}\,d\mathcal{V}(x)}\bigg|\int_{\mathcal{M}} f(x)p(x)^{1+2/m}\,d\mathcal{V}(x)=0\bigg\}.
\end{equation}
The restriction to smooth functions in \eqref{lambda1} (also in Theorem \ref{thm:consistency} and \eqref{endelta}) is not a matter since the associated continuum eigenfunctions are smooth \cite{anderson2004boundary}. Moreover, it can be seen from \eqref{lambda1} that the first (non-constant) eigenfunctions must change their signs, and because in the energy $\mathcal{E}^{\mathcal{M}}(f)$, $|\nabla f|^2$ is weighted against a positive power of $p$, they must only change their signs in low density regions. This gives an advantage in clustering with a few labeled points where one assumes that the classifier remains relatively constant in high density \cite{bousquet2003measure}.\\

Similarly to $\Delta^{snn}$, $\mathcal{L}^{snn}$ is a positive semidefinite operator on $L^2(\mu_{n})$; indeed
\begin{equation*}
    \langle \mathcal{L}^{snn} u, u\rangle_{\mu_{n}} = \frac{(\alpha n/k)^{1+2/m}}{2^{m+2}n^2}\sum_{i,j}\frac{N(x_{i},x_{j})}{k}(u(x_{i})-u(x_{j}))^2 \geq 0.
\end{equation*}
Its spectrum can be listed as
\begin{equation}
    0=\lambda^{\Gamma}_1<\lambda^{\Gamma}_2\leq\cdots\leq\lambda^{\Gamma}_{n}
\end{equation}
assuming that $\Gamma$ is connected \cite{spielman2019spectral}. We anticipate the rate at which $\lambda^{\Gamma}_{l}\to\lambda^{\mathcal{M}}_{l}$ to be quantified in a future work.

\subsection{Outline of the proof for Theorem \ref{thm:consistency}}

We prove the convergence $\mathcal{L}^{snn}\to\Delta^{snn}$ depicted in \eqref{thmstatement} by passing it through an evolution:
\begin{equation} \label{evolution}
    \mathcal{L}^{snn}\to \mathcal{L}^1 \to \mathcal{L}^2 \to \Delta^{snn}.
\end{equation}
The precise definitions of $\mathcal{L}^1$, $\mathcal{L}^2$ will be given in \ref{sec:mainproof}. We treat these two operators and $\Delta^{snn}$ as operators on $L^2(\mu)$ as well as on $L^2(\mu_{n})\cap C(\mathcal{M})$, i.e., we restrict them to the SNN graph $\Gamma$. Although we can treat the graph Laplacian $\mathcal{L}^{snn}$ the same way, we mainly treat it as an operator on $L^2(\mu_{n})$. To progress from an operator on $L^2(\mu_{n})$ to that on $L^2(\mu)$, we will need a key concentration ingredient, given in \ref{sec:conmeas}. Additionally, we will need tools from differential geometry, which are summarized below. 

\subsubsection{Local Riemannian geometry} \label{sec:Riegeo}

The results here were already presented in \cite{calder2022improved, garcia2020error, burago2015graph}; more details can also be found in \cite{do1992riemannian}.\\
Let $\exp_{x}: T_{x}\mathcal{M}\to\mathcal{M}$ be the Riemannian exponential map at $x\in\mathcal{M}$ and $J_{x}(v)$ denote the Jacobian of $\exp_{x}$ at $v\in B(0,r)\subset T_{x}\mathcal{M}$. The Rauch Comparison Theorem \cite{do1992riemannian} states that the relative distortion of metric by $\exp_{x}$ at $z\in B(0,s)\subset T_{x}\mathcal{M}$ is bounded by $O(K|z|^2)$. Hence
\begin{equation} \label{jacobian}
    (1+CK|v|^2)^{-1}\leq J_{x}(v)\leq 1+CK|v|^2,
\end{equation}
from which it follows that (see \cite{calder2022improved})
\begin{equation} \label{volexch}
    |\mathcal{V}(\mathcal{B}(x,r))-\alpha r^{m}|\leq CKr^{m+2}.
\end{equation}
In addition, if $|x-y|\leq R/2$ then by \cite[Proposition 2]{garcia2020error},
\begin{equation} \label{distexch}
    |x-y|\leq d(x,y)\leq |x-y|+\frac{8}{R^2}|x-y|^3,
\end{equation}
where $d$ denotes the geodesic distance on $\mathcal{M}$. Furthermore, when $0<s<\min\{i_0,K^{-1/2}\}$,
\begin{equation} \label{expmap}
    \exp_{x}: T_{x}\mathcal{M} \supset B(0,s)\to \mathcal{B}(x,s)\subset\mathcal{M}
\end{equation}
is a diffeomorphism and a bi-Lipschitz bijection (\cite[Proposition 1]{garcia2020error}). 

\begin{remark} \label{rem:tilde}
When \eqref{expmap} denotes a diffeormorphism, one can write, say, an integral on $\mathcal{M}$, such as
\begin{equation} \label{int}
    \int_{\mathcal{M}} \eta\bigg(\frac{d(x,y)}{s}\bigg)(f(x)-f(y))\,p(y)d\mathcal{V}(y)
\end{equation}
in normal coordinates. We give a demonstration here. We write $\tilde{f}(\cdot)$ to mean $f(\exp_{x}(\cdot))$, a function on $\mathbb{R}^{m}$ when $f$ is a function on $\mathcal{M}$. Then \eqref{int} becomes
\begin{equation*}
    \int_{\mathcal{M}} \eta\bigg(\frac{d(x,y)}{s}\bigg)(f(x)-f(y))\,p(y)d\mathcal{V}(y)
    = \int_{B(0,s)\subset T_{x}\mathcal{M}} \eta\bigg(\frac{|u|}{s}\bigg)(\tilde{f}(0) - \tilde{f}(u))\tilde{p}(u)J_{x}(u)\,du,
\end{equation*}
where $\exp_{x}(u)=y$, $\exp_{x}(0)=x$ and $\eta$ is as in \eqref{def:sharednn}. It should be noted that Assumption \ref{assumption} guarantees that \eqref{expmap} is a diffeomorphism with $s\leq c(k/n)^{1/m}$, whenever $c$ is small enough.
\end{remark}

\begin{remark} \label{rem:exchnorm} If $f\in C^{l}(\mathcal{B}(x,r))$, then $\|f\|_{C^{l}(\mathcal{B}(x,r))}$ and $\|\tilde{f}\|_{C^{l}(B(0,r))}$ are equivalent. Indeed, following \cite{hebey1996sobolev}, one can write the covariant derivatives of $f$ in normal coordinates around $x$, and use standard expansions for the Christoffel symbols and metric tensor in these coordinates, to conclude that
\begin{equation*}
    (1-Cr)\|\tilde{f}\|_{C^{l}(B(0,r))}\leq \|f\|_{C^{l}(\mathcal{B}(x,r))}\leq (1+Cr)\|\tilde{f}\|_{C^{l}(B(0,r))}.
\end{equation*}
\end{remark}

\subsubsection{A crucial measure concentration result} \label{sec:conmeas}

\begin{lemma} \label{calderlem:conc} \cite[Lemma 3.1 and Proof]{calder2022improved}
Let $\psi:\mathcal{M}\to\mathbb{R}$ be bounded and Borel measurable. For $x\in\mathcal{M}$, define
\begin{equation*} 
    \Psi := \sum_{i:|x_{i}-x|\leq\epsilon} \psi(x_{i}).
\end{equation*}
Then for any $\epsilon^2\leq\delta\leq 1$, we have that
\begin{equation} \label{calderlem1conc}
    |a-b| \leq Cp_{max}\|\psi\|_{\infty} n\epsilon^{m}\delta
\end{equation}
and that
\begin{equation} \label{calderlem2conc}
    \mathbb{P}(|\Psi-a|\geq Cp_{max}\|\psi\|_{\infty}\delta n\epsilon^{m}) \leq 2\exp(-cp_{max}\delta^2 n\epsilon^{m}).
\end{equation}
where $\|\psi\|_{\infty}:=\|\psi\|_{L^{\infty}(\mathcal{B}(x,2\epsilon))}$ and
\begin{equation*}
    a := n\int_{\mathcal{B}(x,\epsilon)} \psi(z)p(z)\,d\mathcal{V}(z) \quad\text{ and }\quad
    b := n\int_{B(x,\epsilon)\cap\mathcal{M}} \psi(z)p(z)\,d\mathcal{V}(z).
\end{equation*}
\end{lemma}

\section{Main proof} \label{sec:mainproof}

Define, for $x,y\in\mathcal{M}$, 
\begin{align*}
    w_{k}(x,y) := \int_{\mathcal{M}} \eta\bigg(\frac{d(x,z)}{\varepsilon_{k}(x)}\bigg)\eta\bigg(\frac{d(y,z)}{\varepsilon_{k}(y)}\bigg) p(z)\,d\mathcal{V}(z).
\end{align*}
We claim that $N(x,y)\approx nw_{k}(x,y)$ with high probability.

\begin{lemma} \label{lem:weight1}
Let $x,y\in\mathcal{M}$. For every $\varepsilon_{k}(x,y)^2\leq\delta\leq 1$, 
\begin{equation*}
    \mathbb{P}(|N(x,y)-nw_{k}(x,y)|\geq Cp_{max}\delta n\varepsilon_{k}(x,y)^{m})\leq 2\exp(-cp_{max}\delta^2 n\varepsilon_{k}(x,y)^{m}),
\end{equation*}
where $\varepsilon_{k}(x,y):=\varepsilon_{k}(x)+\varepsilon_{k}(y)$.
\end{lemma}

The proof of Lemma \ref{lem:weight1} is a direct consequence of Lemma \ref{calderlem:conc} and is given in the Appendix \ref{sec:weight1}. This inspires us to define
\begin{equation} \label{def:L1}
    \mathcal{L}^1 u(x) := \frac{(\alpha n/k)^{1+2/m}}{2^{m+2} k}\sum_{i=1}^{n} w_{k}(x,x_{i})(u(x)-u(x_{i})).
\end{equation}
Next, for every $x\in\mathcal{M}$, we let (see also \cite{calder2022improved})
\begin{equation} \label{def:eps}
    k/n =: \alpha p(x)\varepsilon(x)^{m}. 
\end{equation}
Since $p\in C^2(\mathcal{M})$ and is bounded away from zero, $\varepsilon\in C^2(\mathcal{M})$ and therefore is a continuous version of $\varepsilon_{k}$. Let $\varepsilon^{\star}:=\max\{\varepsilon(x):x\in\mathcal{M}\}$. Then $\varepsilon(x),\varepsilon^{\star}\asymp_{p} (k/n)^{1/m}$; in particular, it follows from Assumption \ref{assumption} that \eqref{expmap} holds for $s=3\varepsilon^{\star}$. \\
As an operator on $L^2(\mu_{n})$, $\mathcal{L}^1$ predicts the average behavior of $\mathcal{L}^{snn}$; this is the content of the following lemma.

\begin{lemma} \label{lem:firstevol}
Let $u\in L^2(\mu_{n})$ such that $u=f|_{X}$ for some $f\in C^{0,1}(\mathcal{M})$. Let $x\in X$. Then
\begin{equation*}
    \mathbb{P}(|\mathcal{L}^{snn}u(x)-\mathcal{L}^1u(x)|\geq C_{p} [f]_{1;\mathcal{B}(x,2\varepsilon^{\star})}t)\leq Cn\exp(-c_{p} k (k/n)^{2/m}t^2)
\end{equation*}
whenever $(k/n)^{1/m}\leq t\leq (k/n)^{-1/m}$.
\end{lemma}

Here, $C^{0,1}(\mathcal{M})$ denotes the space of $1$-H\"older continuous \cite{folland1999real} functions on $\mathcal{M}$, and 
\begin{equation*}
    [f]_{1;\mathcal{B}(x,2\varepsilon^{\star})} :=\sup_{y\not= z\in\mathcal{B}(x,2\varepsilon^{\star})} \frac{|f(y)-f(z)|}{d(y,z)}.
\end{equation*}

\begin{proof}[Proof of Lemma \ref{lem:firstevol}] By Lemma \ref{calderlem:epsk}, if $C(k/n)^{2/m}\leq\delta\leq 1$, then
\begin{align} 
    \nonumber |\alpha p(x)\varepsilon_{k}(x)^{m}-k/n| &\leq C\delta k/n\\
    \label{epskbds} \max_{1\leq i\leq n} |\alpha p(x_{i})\varepsilon_{k}(x_{i})^{m}-k/n| &\leq C\delta k/n
\end{align}
happen with probability at least $1-Cn\exp(-c\delta^2 k)$. Hence, we can assume \eqref{epskbds} holds, which implies 
\begin{equation} \label{epskbds1}
    \varepsilon_{k}(x),\max_{1\leq i\leq n}\varepsilon_{k}(x_{i})\leq C_{p}(k/n)^{1/m}.
\end{equation}
Take $C_{p}(k/n)^{2/m}\leq\delta\leq 1$, and consider
\begin{equation*}
    \mathcal{I}(x) :=card(\{i: N(x,x_{i})>0\}) \quad\text{ and }\quad \mathcal{J}(x) :=card(\{i: w_{k}(x,x_{i})>0\}).
\end{equation*}
By definition \eqref{def:sharednn} and \eqref{epskbds1}, $N(x,x_{i})>0$ only if
\begin{equation*}
    |x-x_{i}|\leq\varepsilon_{k}(x,x_{i})\leq C_{p}(k/n)^{1/m}
\end{equation*}
which, by Lemma \ref{calderlem:N} and our choice of $\delta$, implies that $\mathcal{I}(x)\leq C_{p}k$. Similarly, since $p>0$, $w_{k}(x,x_{i})>0$ iff 
\begin{equation*}
    d(x,x_{i})<\varepsilon_{k}(x,x_{i})\leq C_{p}(k/n)^{1/m}.
\end{equation*}
This last bit and \eqref{distexch} yield that $|x-x_{i}|\leq C_{p}(k/n)^{1/m}$. By Lemma \ref{calderlem:N} again, $\mathcal{J}(x)\leq C_{p}k$. All this together with Lemma \ref{lem:weight1} and \eqref{distexch}, \eqref{epskbds1}, leads to
\begin{equation} \label{LsL1diff} 
    |\mathcal{L}^{snn}u(x)-\mathcal{L}^1u(x)| \leq \frac{(\alpha n/k)^{1+2/m}}{2^{m+2} n}\sum_{i=1}^{n}\bigg|\frac{N(x,x_{i})}{k}-\frac{n w_{k}(x,x_{i})}{k}\bigg|\cdot|u(x)-u(x_{i})|
    \leq C_{p}\delta[f]_{1;\mathcal{B}(x,\varepsilon^{\star})}\bigg(\frac{k}{n}\bigg)^{-1/m},
\end{equation}
with probability at least $1-2\exp(-c_{p}\delta^2k)$. Let $\delta$ take the form $C_{p} (k/n)^{1/m}t$, where $(k/n)^{1/m}\leq t\leq (k/n)^{-1/m}$. Putting this back in \eqref{LsL1diff}, we obtain
\begin{equation*}
    |\mathcal{L}^{snn}u(x)-\mathcal{L}^1u(x)|\leq C_{p}[f]_{1;\mathcal{B}(x,\varepsilon^{\star})}t.
\end{equation*}
Then combining the events described in Lemma \ref{lem:weight1} and \eqref{epskbds} gives the desired conclusion.
\end{proof}

Define the operator $\mathcal{L}^2$ on $L^2(\mu)$ in \eqref{evolution} to be
\begin{equation*}
    \mathcal{L}^2 f(x) := \frac{(\alpha n/k)^{2+2/m}}{\alpha 2^{m+2}}\int_{\mathcal{B}(x,2\varepsilon^{\star})} w(x,y)(f(x)-f(y))p(y)\,d\mathcal{V}(y)
\end{equation*}
where
\begin{equation} \label{def:weightL2}
    w(x,y):=\int_{\mathcal{M}} \eta\bigg(\frac{d(x,z)}{\varepsilon(x)}\bigg)\eta\bigg(\frac{d(y,z)}{\varepsilon(y)}\bigg) p(z)\,d\mathcal{V}(z).
\end{equation}
In a similar fashion to Lemma \ref{lem:firstevol}, we claim that, as an operator on $L^2(\mu)\cap C^{0,1}(\mathcal{M})$, $\mathcal{L}^2$ closely describes $\mathcal{L}^1$.

\begin{lemma} \label{lem:midevol}
Let $f\in C^{0,1}(\mathcal{M})$ and $x\in\mathcal{M}$. Then
\begin{equation*}
    \mathbb{P}(|\mathcal{L}^1f(x)-\mathcal{L}^2f(x)|\geq C_{p}[f]_{1;\mathcal{B}(x,2\varepsilon^{\star})}t)\leq C_{p}n\exp(-c_{p} k (k/n)^{2/m}t^2)
\end{equation*}
with $ (k/n)^{1/m}\leq t\leq (k/n)^{-1/m}$.
\end{lemma}
The proof of Lemma \ref{lem:midevol} is similar to that of Lemma \ref{lem:firstevol} but a tad more involved and therefore is given in the Appendix \ref{sec:midevol}. We now define
\begin{equation*}
    \omega(x,y) := (\alpha n/k)w(x,y) = (\alpha n/k)\int_{\mathcal{M}} \eta\bigg(\frac{d(x,z)}{\varepsilon(x)}\bigg)\eta\bigg(\frac{d(y,z)}{\varepsilon(y)}\bigg)p(z)\,d\mathcal{V}(z).
\end{equation*}
If $y\in\mathcal{B}(x,2\varepsilon^{\star})$ then by Assumption \ref{assumption} and \eqref{expmap}, we can write $\omega(x,y) = \tilde{\omega}(0,v)$, where $v\in B(0,2\varepsilon^{\star})$ and $\exp_{x}(v)=y$. Denote $\tilde{\omega}_0(v) := \tilde{\omega}(0,v)$. For the same reason, we can write 
\begin{align}
    \nonumber \tilde{\omega}_0(v) = \omega(x,y) &= (\alpha n/k)\int_{\mathcal{M}} \eta\bigg(\frac{d(x,z)}{\varepsilon(x)}\bigg)\eta\bigg(\frac{d(y,z)}{\varepsilon(y)}\bigg) p(z)\,d\mathcal{V}(z)\\
    \label{omega0} &= (\alpha n/k)\int_{B(0,2\varepsilon(x,y))} \eta\bigg(\frac{|u|}{\tilde{\varepsilon}(0)}\bigg)\eta\bigg(\frac{|u-v|}{\tilde{\varepsilon}(v)}\bigg) \tilde{p}(u)J_{x}(u)\,du.
\end{align}
where $\varepsilon(x,y):=\varepsilon(x)+\varepsilon(y)$. Note that $\tilde{\varepsilon}(0)=\varepsilon(x)$ and $\tilde{\varepsilon}(v)=\varepsilon(y)$. It's clear from \eqref{omega0} that $\tilde{\omega}_0$ is continuous. Furthermore, it will be proved in the Appendix \ref{sec:frakw} that $\tilde{\omega}_0\in C^2(B(0,2\varepsilon^{\star}))$ and that $\|\tilde{\omega}_0\|_{C^2(B(0,2\varepsilon^{\star}))}=O(1)$. Moreover, 
\begin{equation} \label{omegader}
    \nabla\tilde{\omega}_0(0) = 0.
\end{equation}
These facts will be needed in the proof of Proposition \ref{prop:graphcons} below, which is the final ingredient to the proof of Theorem \ref{thm:consistency}.

\begin{proposition} \label{prop:graphcons}
Let $f\in C^3(\mathcal{M})$ and $x\in\mathcal{M}$. Then the following holds
\begin{equation*}
    |\mathcal{L}^2f(x) - \frac{1}{m+2}\Delta^{snn} f(x)|\leq C_{p}(1+\|f\|_{C^3(\mathcal{B}(x,2\varepsilon^{\star}))})\varepsilon(x).
\end{equation*}
\end{proposition}

In what follows and the remaining of this paper, we will use the big $O$ notation to indicate a quantity whose magnitude is bounded by a constant multiple of what is inside the brackets, and this constant is allowed to depend on various factors at play, such as the intrinsic values mentioned in \ref{sec:setup} as well as $\|f\|_{C^3(\mathcal{M})}$ and $\|p\|_{C^2(\mathcal{M})}$.

\begin{proof}[Proof of Proposition \ref{prop:graphcons}]
For simplicity in presentation, we set $\varepsilon:=\varepsilon(x)$ here. Note by definition that
\begin{equation*}
    \alpha\mathcal{L}^2 f(x) = \frac{(\alpha n/k)^{1+2/m}}{2^{m+2}}\int_{\mathcal{B}(x,2\varepsilon^{\star})} \eta\bigg(\frac{d(x,y)}{\varepsilon+\varepsilon(y)}\bigg)\omega(x,y)(f(x)-f(y))p(y)\,d\mathcal{V}(y).
\end{equation*}
Then by the change in variables given by \eqref{expmap} (see Remark \ref{rem:tilde}),
\begin{align} 
    \nonumber \alpha\mathcal{L}^2 f(x) &= -\frac{(\alpha n/k)^{1+2/m}}{2^{m+2}}\int_{B(0,2\varepsilon^{\star})\subset T_{x}\mathcal{M}} \eta\bigg(\frac{|v|}{\varepsilon + \tilde{\varepsilon}(v)}\bigg)\tilde{\omega}_0(v)(\tilde{f}(v)-\tilde{f}(0))\tilde{p}(v)J_{x}(v)\,dv\\
    \label{Lap2} &= -\frac{(\alpha n/k)^{2/m}}{4\tilde{p}(0)}\int_{B(0,C^{\star})} \eta(|v|[(1/2)(1+s(2\varepsilon v)]^{-1})\tilde{\omega}_0(2\varepsilon v)(\tilde{f}(2\varepsilon v)-\tilde{f}(0))\tilde{p}(2\varepsilon v)J_{x}(2\varepsilon v)\,dv
\end{align}
where $s(v) := \tilde{\varepsilon}(v)/\varepsilon = (\tilde{p}(0)/\tilde{p}(v))^{1/m}$ and $C^{\star}:=\varepsilon^{\star}/\varepsilon$. By the Taylor's expansions, 
\begin{align*}
    s(2\varepsilon v) &= 1 - \frac{2\varepsilon}{m}\nabla\log\tilde{p}(0)\cdot v + O(\varepsilon^2)\\
    \tilde{p}(2\varepsilon v) &= \tilde{p}(0) + 2\nabla\tilde{p}(0)\cdot\varepsilon v+O(\varepsilon^2)\\
    \tilde{f}(2\varepsilon v)-\tilde{f}(0) &= 2\varepsilon\nabla\tilde{f}(0)\cdot v +2\varepsilon^2 v\cdot\nabla^2\tilde{f}(0) v+ O(\varepsilon^3).
\end{align*}
Similarly, by \eqref{omegader} and that $\|\tilde{\omega}_0\|_{C^2(B(0,2\varepsilon^{\star}))}=O(1)$, we also have
\begin{equation*}
    \tilde{\omega}_0(2\varepsilon v) = \tilde{\omega}_0(0) + O(\varepsilon^2).
\end{equation*}
Putting these back in \eqref{Lap2}, and recalling \eqref{jacobian}, we can write $\alpha\mathcal{L}^2 f(x) = {\bf L} f(x) + O(\varepsilon)$, where
\begin{multline*}
    {\bf L} f(x) := -\frac{\varepsilon(\alpha n/k)^{2/m}\tilde{\omega}_0(0)}{2}\int_{B} \eta(|v|[1-(\varepsilon/m)\nabla\log\tilde{p}(0)\cdot v)]^{-1})\\
    (\nabla\tilde{f}(0)\cdot v +\varepsilon v\cdot\nabla^2\tilde{f}(0) v) (1+2\nabla\log\tilde{p}(0)\cdot\varepsilon v)\,du
\end{multline*}
where
\begin{equation*}
    B :=\{v\in\mathbb{R}^{m}: |v|\leq 1-\frac{\varepsilon}{m}\nabla\log\tilde{p}(0)\cdot v\}. 
\end{equation*}
Consider the change in variables
\begin{equation*}
    v\mapsto z=\Phi(v):=\frac{v}{1-\frac{\varepsilon}{m}\nabla\log\tilde{p}(0)\cdot v}.
\end{equation*}
Then when $\varepsilon$ is small enough, $\Phi$ is invertible and
\begin{equation} \label{vtoz}
    v = \Phi^{-1}(z) = z(1-\frac{\varepsilon}{m}\nabla\log\tilde{p}(0)\cdot z) + O(\varepsilon^2). 
\end{equation}
Therefore $D\Phi^{-1}(v) = I - \frac{\varepsilon}{m}((\nabla\log\tilde{p}(0)\cdot z)I+ \nabla\log\tilde{p}(0)\otimes z)$. Utilizing the fact that  $det(I+\varepsilon X) = 1 +\varepsilon Tr(X) + O(\varepsilon^2)$, we obtain
\begin{equation} \label{detvir}
    |det(D\Phi^{-1}(z))| = |1 - \varepsilon(1+\frac{1}{m})(\nabla\log\tilde{p}(0)\cdot z)| + O(\varepsilon^2),
\end{equation}
which, by the second point in Assumption \ref{assumption}, becomes
\begin{equation} \label{det}
    |det(D\Phi^{-1}(z))| = 1 - \varepsilon(1+\frac{1}{m})(\nabla\log\tilde{p}(0)\cdot z) + O(\varepsilon^2)
\end{equation}
for $|z|\leq 1$. We apply \eqref{vtoz}, \eqref{det} to ${\bf L}f(x)$ and transform its integral to
\begin{multline*} 
    -\frac{\varepsilon (\alpha n/k)^{2/m}\tilde{\omega}_0(0)}{2}\int_{B(0,1)} 
    (\nabla\tilde{f}(0)\cdot z-\frac{\varepsilon}{m}(\nabla\log\tilde{p}(0)\cdot z)(\nabla\tilde{f}(0)\cdot z) +\varepsilon z\cdot\nabla^2\tilde{f}(0) z)\\ (1+2\nabla\log\tilde{p}(0)\cdot\varepsilon z)(1 - \varepsilon(1+\frac{1}{m})(\nabla\log\tilde{p}(0)\cdot z))\,dz + O(\varepsilon),
\end{multline*}
which, through simple calculations, yields
\begin{equation} \label{bfL13}
    {\bf L} f(x)= -\frac{\sigma\tilde{\omega}_0(0)}{2\tilde{p}(0)^{2/m}} \bigg((1-2/m)\nabla\log\tilde{p}(0)\cdot\nabla\tilde{f}(0) + \Delta\tilde{f}(0)\bigg) + O(\varepsilon),
\end{equation}
where, recall that $\sigma=\int_{B(0,1)}|u_1|^2\eta(|u|)\,du$. It will be shown in the Appendix \ref{sec:frakw} that $\tilde{\omega}_0(0) = \alpha + O(\varepsilon)$; putting this in \eqref{bfL13} and recalling that $\alpha\mathcal{L}^2 f(x) = {\bf L}f(x) + O(\varepsilon^2)$ deliver us
\begin{align*}
    \mathcal{L}^2 f(x) &= -\frac{\tilde{p}(0)^{-2/m}}{2(m+2)}\bigg((1-2/m)\nabla\tilde{f}(0)\cdot\nabla\log\tilde{p}(0) + \Delta\tilde{f}(0)\bigg) + O(\varepsilon)\\
    &= -\frac{1}{2(m+2)\tilde{p}(0)}div(\tilde{p}^{1-2/m}\nabla\tilde{f})(0) + O(\varepsilon) = -\frac{1}{m+2}\Delta^{snn}f(x) + O(\varepsilon),
\end{align*}
since $\sigma=\alpha/(m+2)$. By Remark \ref{rem:exchnorm}, $\|\tilde{f}\|_{C^3(B(0,2\varepsilon^{\star}))}$ is equivalent to $\|\tilde{f}\|_{C^3(\mathcal{B}(x,2\varepsilon^{\star}))}$; hence the proof is now finished. 
\end{proof}

\begin{remark} \label{rem:secondpt}
Note from \eqref{detvir} that in place of $c'_{p,\mathcal{M}}$ in Assumption \ref{assumption}, one can take any estimate upper bound of $\max_{x\in\mathcal{M}}|\nabla p(x)|$.
\end{remark}

\begin{proof}[Proof of Theorem \ref{thm:consistency}]
Let $f\in C^3(\mathcal{M})$. Observe from the proof of Lemma \ref{lem:firstevol} that \eqref{LsL1diff} follows for any $x\in X$ once the uniformity condition \eqref{epskbds} is satisfied, and similarly, from the proof of Lemma \ref{lem:midevol} that \eqref{firstsum}, \eqref{secondsum} hold once \eqref{epsassump} does. Hence as a result of these lemmas, if $(k/n)^{1/m}\leq t\leq (k/n)^{-1/m}$,
\begin{equation} \label{31}
    \max_{1\leq i\leq n} |\mathcal{L}^{snn}f(x_{i})-\mathcal{L}^2f(x_{i})|\leq C_{p}\|f\|_{C^1(\mathcal{M})}t
\end{equation}
with probability at least $1-C_{p}n\exp(-c_{p}k(k/n)^{2/m}t^2)$. It also follows from Proposition \ref{prop:graphcons} that,
\begin{equation} \label{32}
    \max_{1\leq i\leq n} |\mathcal{L}^2 f(x_{i})-\frac{1}{m+2}\Delta^{snn}f(x_{i})|\leq C_{p}(1+\|f\|_{C^3(\mathcal{M})})(k/n)^{1/m}\leq C_{p}(1+\|f\|_{C^3(\mathcal{M})})t.
\end{equation}
Combining \eqref{31}, \eqref{32}, we obtain the theorem conclusion.
\end{proof}

\section{Conclusion}
In this paper, we start the theoretical study of Laplacians based on SNN graphs, which was not present in the literature. We establish a consistency result and show that the large scale asymptotics of unnormalized SNN graph Laplacian reach the same weighted Laplace-Beltrami limit operator as those of unnormalized $k$-NN graph Laplacian. \\
Some directions for future investigation, that we plan to follow up this work with, are: (1) analyzing the spectral convergence of SNN graph Laplacian, (2) analyzing the effectiveness of spectral clustering methods using SNN graphs. In the second direction, our intention is to include the mixed distribution case, $\mu=\sum_{i=1}^{K}\mu_{i}$, in which density can vanish on the manifold. It is also of interest to extend the analysis to the case of compact manifolds with boundaries. 

\section{Appendix} \label{appx}

\subsection{Proof of Lemma \ref{lem:weight1}} \label{sec:weight1}

Denote $\mathcal{B}_{k}(x):=\mathcal{B}(x,\varepsilon_{k}(x))$ and recall that $B_{k}(x)=B(x,\varepsilon_{k}(x))$. Define also
\begin{equation*}
    B_{k}(x)^{o} :=\{ z\in\overline{B_{k}(x)}: z\not= x\}.
\end{equation*}
We apply \eqref{calderlem2conc} in Lemma \ref{calderlem:conc} to $N(x,y)$ in \eqref{def:sharednn} - with $\psi(z)= 1_{B_{k}(y)^{o}}(z)$ and 
\begin{equation*} 
    N(x,y) = \sum_{i: 0<|x_{i}-x|\leq\varepsilon_{k}(x)} \psi(x_{i}) = \sum_{i: 0<|x_{i}-x|\leq\varepsilon_{k}(x)} 1_{B_{k}(y)^{o}}(x_{i}),
\end{equation*}
to get, for every $\varepsilon_{k}(x)^2\leq\delta\leq 1$, 
\begin{multline} \label{fact1}
    \bigg|N(x,y)-n\int_{\mathcal{B}_{k}(x)} 1_{B_{k}(y)}(z) p(z)\,d\mathcal{V}(z)\bigg| \\
    =\bigg|N(x,y)-n\int_{B_{k}(y)\cap\mathcal{M}} 1_{\mathcal{B}_{k}(x)}(z)p(z)\,d\mathcal{V}(z)\bigg| \leq Cp_{max}\delta n\varepsilon_{k}(x)^{m},
\end{multline}
which happens with probability at least $1-2\exp(-cp_{max}\delta^2n\varepsilon_{k}(x)^{m})$. Moreover, since $1_{\mathcal{B}_{k}(x)}(z)$ is a measurable function on $\mathcal{M}$, it follows from \eqref{calderlem1conc} in Lemma \ref{calderlem:conc} that
\begin{equation} \label{fact2}
    \bigg|\int_{B_{k}(y)\cap\mathcal{M}} 1_{\mathcal{B}_{k}(x)}(z) p(z)\,d\mathcal{V}(z)-\int_{\mathcal{B}_{k}(y)} 1_{\mathcal{B}_{k}(x)}(z) p(z)\,d\mathcal{V}(z)\bigg| \leq Cp_{max}\varepsilon_{k}(y)^{m}\delta.
\end{equation}
Combining \eqref{fact1}, \eqref{fact2} yields us
\begin{equation*}
    |N(x,y)-nw_{k}(x,y)|\leq Cp_{max}\delta n(\max\{\varepsilon_{k}(x),\varepsilon_{k}(y)\})^{m}=Cp_{max}\delta n\varepsilon_{k}(x,y)^{m},
\end{equation*}
with probability at least $1-2\exp(-cp_{max}\delta^2n\varepsilon_{k}(x,y)^{m})$. \qed

\subsection{Proof of Lemma \ref{lem:midevol}} \label{sec:midevol}

We define an intermediate operator between $\mathcal{L}^1$ and $\mathcal{L}^2$ as
\begin{equation*}
    \mathcal{L}^{\sharp}f(x) :=\frac{(\alpha n/k)^{1+2/m}}{2^{m+2}k}\sum_{i=0}^{n} w(x,x_{i})(f(x)-f(x_{i}))
\end{equation*}
where $w$ is as in \eqref{def:weightL2}. Suppose the following holds. 

\begin{lemma} \label{lem:secondevol}
Let $f\in C^{0,1}(\mathcal{M})$ and $x\in\mathcal{M}$. Then
\begin{equation*}
    \mathbb{P}(|\mathcal{L}^1f(x)-\mathcal{L}^{\sharp}f(x)|\geq C_{p}[f]_{1;\mathcal{B}(x,2\varepsilon^{\star})}t)\leq C_{p}n\exp(-c_{p} k (k/n)^{2/m}t^2)
\end{equation*}
with $ (k/n)^{1/m}\leq t\leq  (k/n)^{-1/m}$.
\end{lemma}

Observe from definition that $|w(x,\cdot)|\leq Cp_{max}k/n$. Hence a routine application of Lemma \ref{calderlem:conc} to $\mathcal{L}^{\sharp}f(x)$ yields
\begin{equation*} 
    \mathbb{P}(|\mathcal{L}^{\sharp}f(x) - \mathcal{L}^2f(x)|\geq C_{p}[f]_{1;\mathcal{B}(x,2\varepsilon^{\star})}t)\leq 2\exp(-c_{p}k (k/n)^{2/m}t^2)
\end{equation*}
for $(k/n)^{1/m}\leq t\leq (k/n)^{-1/m}$, which together with Lemma \ref{lem:secondevol}, concludes Lemma \ref{lem:midevol}. \qed

\subsubsection{Proof of Lemma \ref{lem:secondevol}}

Recall that $\varepsilon_{k}(x,y)=\varepsilon_{k}(x)+\varepsilon_{k}(y)$ and $\varepsilon(x,y)=\varepsilon(x)+\varepsilon(y)$. Fix $x\in\mathcal{M}$, and let
\begin{align*}
    A &:= \{i: |x-x_{i}|< \varepsilon_{k}(x,x_{i})\}\\
    \tilde{A}(s) &:= \{i: |x-x_{i}|< \varepsilon(x,x_{i})(1+s)\},
\end{align*}
for $s\geq 0$. We can rewrite $\mathcal{L}^1f(x)$, $\mathcal{L}^{\sharp}f(x)$ in terms of these sets, as follows,
\begin{align*}
    \mathcal{L}^1f(x) &= \frac{(\alpha n/k)^{1+2/m}}{2^{m+2}n}\sum_{i\in A} \frac{n w_{k}(x,x_{i})}{k}(f(x)-f(x_{i}))\\
    \mathcal{L}^{\sharp}f(x) &= \frac{(\alpha n/k)^{1+2/m}}{2^{m+2}n}\sum_{i\in \tilde{A}(0)} \frac{n w(x,x_{i})}{k}(f(x)-f(x_{i})).
\end{align*}
We will show that, although $A$, $\tilde{A}(0)$ are two different sets, they are roughly comparable with respect to set inclusion. To see this, note that by Lemma \ref{calderlem:epsk} and \eqref{def:eps}, we can assume
\begin{align}
    \nonumber |\varepsilon_{k}(x)^{m}-\varepsilon(x)^{m}| &\leq C\delta\varepsilon(x)^{m} \\
    \label{epsassump} \max_{1\leq i\leq n} |\varepsilon_{k}(x_{i})^{m}-\varepsilon(x_{i})^{m}| &\leq C\delta\varepsilon(x_{i})^{m}
\end{align}
for some fixed $C(k/n)^{2/m}\leq\delta\leq 1$. It follows that
\begin{equation*}
    (1-C\delta)\varepsilon(x,x_{i})\leq \varepsilon_{k}(x,x_{i})\leq (1+C\delta)\varepsilon(x,x_{i}),
\end{equation*}
and hence
\begin{equation} \label{Ainclusion}
    \tilde{A}(-C\delta)\subset A, \tilde{A}(0) \subset \tilde{A}(C\delta).
\end{equation}
Therefore
\begin{equation} \label{L2L1}
    |\mathcal{L}^{\sharp}u(x)-\mathcal{L}^1u(x)|\leq \frac{(\alpha n/k)^{1+2/m}}{2^{m+2}n}\bigg[\sum_{i\in \tilde{A}(C\delta)\setminus\tilde{A}(-C\delta)} +  \sum_{i\in\tilde{A}(-C\delta)} \frac{n|w_{k}(x,x_{i})-w(x,x_{i})|}{k} |f(x)-f(x_{i})|\bigg].
\end{equation}
To handle the first sum, we utilize the Chernoff's bounds - and recall \eqref{def:eps} again - to obtain that
\begin{equation} \label{chern}
    \mathbb{P}(card(\tilde{A}(C\delta))- card(\tilde{A}(-C\delta))\geq C\delta n\varepsilon(x)^{m})\leq 2\exp(-c_{p}\delta^2 k).
\end{equation}
It follows from definition and \eqref{epsassump} that $|w_{k}(x,x_{i})-w(x,x_{i})|\leq C_{p}k/n$; using this and \eqref{epsassump}, \eqref{chern}, we have the first sum in \eqref{L2L1} dominated by
\begin{equation} \label{firstsum}
    \frac{(\alpha n/k)^{1+2/m}}{2^{m+2}n}\sum_{i\in \tilde{A}(C\delta)\setminus\tilde{A}(-C\delta)} \frac{n|w_{k}(x,x_{i})-w(x,x_{i})|}{k}|f(x)-f(x_{i})|
   \leq C_{p}\bigg(\frac{k}{n}\bigg)^{-1/m}\delta[f]_{1;\mathcal{B}(x,2\varepsilon^{\star})},
\end{equation}
with probability at least $1-2\exp(-c_{p}\delta^2 k)$. To handle the second sum in \eqref{L2L1}, we use the following lemma. 

\begin{lemma} \label{lem:twoweights}
Suppose \eqref{epsassump} holds. Then for every $x\in\mathcal{M}$, $i\in\tilde{A}(-C\delta)$, 
\begin{equation} \label{lem2wconc}
    |w_{k}(x,x_{i})-w(x,x_{i})|\leq C_{p}\delta\bigg(\frac{k}{n}\bigg).
\end{equation}
\end{lemma}

By invoking Lemma \ref{calderlem:N} and \eqref{Ainclusion}, we obtain that $card(\tilde{A}(-C\delta))\leq C_{p}k$. Now specifying Lemma \ref{lem:twoweights} to the second sum in \eqref{L2L1} yields
\begin{equation} \label{secondsum}
    \frac{(\alpha n/k)^{1+2/m}}{2^{m+2}n}\sum_{i\in\tilde{A}(-C\delta)} \frac{n|w_{k}(x,x_{i})-w(x,x_{i})|}{k}|f(x)-f(x_{i})|\leq C_{p}\bigg(\frac{k}{n}\bigg)^{-1/m}\delta[f]_{1;\mathcal{B}(x,2\varepsilon^{\star})}.
\end{equation}
We combine \eqref{firstsum}, \eqref{secondsum} and let $\delta=C (k/n)^{1/m}t$ for some $ (k/n)^{1/m}\leq t\leq (k/n)^{-1/m}$ to obtain the conclusion of Lemma \ref{lem:secondevol}. \qed\\

\noindent {\it Proof of Lemma \ref{lem:twoweights}.} Note that the condition $i\in\tilde{A}(-C\delta)$ and \eqref{Ainclusion} make both $w_{k}(x,x_{i})$, $w(x,x_{i})$ nonzero in \eqref{lem2wconc}; if not, the bound there can be as large as $C_{p}k/n$. Denote $\mathcal{B}(x):=\mathcal{B}(x,\varepsilon(x))$. Suppose $\varepsilon_{k}(x)\leq\varepsilon(x)$. We write $w_{k}(x,x_{i})-w(x,x_{i})$ as the following sum
\begin{multline} \label{weightdifference}
    \bigg(\int_{\mathcal{B}_{k}(x)} 1_{\mathcal{B}_{k}(x_{i})}(z)p(z)\,d\mathcal{V}(z) - \int_{\mathcal{B}(x)} 1_{\mathcal{B}_{k}(x_{i})}(z)p(z)\,d\mathcal{V}(z)\bigg)\\
    + \bigg(\int_{\mathcal{B}_{k}(x_{i})} 1_{\mathcal{B}(x)}(z)p(z)\,d\mathcal{V}(z) - \int_{\mathcal{B}(x_{i})} 1_{\mathcal{B}(x)}(z)p(z)\,d\mathcal{V}(z)\bigg).
\end{multline}
We express \eqref{weightdifference} in normal coordinates, using the convention established in Remark \ref{rem:tilde}:
\begin{multline} \label{wdiff1}
    \int_{\mathcal{B}(x)} 1_{\mathcal{B}_{k}(x_{i})}(z)p(z)\,d\mathcal{V}(z)-\int_{\mathcal{B}_{k}(x)} 1_{\mathcal{B}_{k}(x_{i})}(z)p(z)\,d\mathcal{V}(z)\\
    =\int_{B(0,\varepsilon(x))\subset T_{x}\mathcal{M}} \eta\bigg(\frac{|u-v|}{\varepsilon_{k}(x_{i})}\bigg)\bigg(\eta\bigg(\frac{|u|}{\varepsilon(x)}\bigg) - \eta\bigg(\frac{|u|}{\varepsilon_{k}(x)}\bigg)\bigg) \tilde{p}(u)J_{x}(u)\,du
\end{multline}
where we've centered at $x$, and so $\exp_{x}(0)=x$, $\exp_{x}(v)=x_{i}$, $\exp_{x}(u)=z$. It follows from this and \eqref{epsassump} that
\begin{align*}
    \bigg|\int_{\mathcal{B}_{k}(x)} 1_{\mathcal{B}_{k}(x_{i})}(z)p(z)\,d\mathcal{V}(z)-\int_{\mathcal{B}(x)} 1_{\mathcal{B}_{k}(x_{i})}(z)p(z)\,d\mathcal{V}(z)\bigg| &\leq Cp_{max}|Vol(B(0,\varepsilon_{k}(x))-Vol(B(0,\varepsilon(x))|\\
    &\leq Cp_{max}\delta\varepsilon(x)^{m}\leq C_{p}\delta(k/n).
\end{align*}
Note that \eqref{wdiff1} corresponds to the first term in \eqref{weightdifference}; the second term can be handled similarly, and so we conclude the lemma. \qed

\subsection{Properties of $\tilde{\omega}_0$} \label{sec:frakw}

Once again, we denote $\varepsilon:=\tilde{\varepsilon}(0)=\varepsilon(x)$. It's clear from \eqref{omega0}, \eqref{jacobian} that
\begin{equation*}
    \tilde{\omega}_0(0) = (\alpha n/k)\int_{B(0,\varepsilon)} \eta\bigg(\frac{|u|}{\varepsilon}\bigg) \tilde{p}(u)J_{x}(u)\,du = \int_{B(0,1)} \,du + O(\varepsilon) = \alpha + O(\varepsilon). 
\end{equation*}

\noindent {\it Derivatives.} For $i=1,\cdots,m$, let $\partial_{i}$ denote the partial derivative along the $i$th axis. Observe that, for every $v\in B(0,2\varepsilon^{\star})$ such that $v=v_{i}e_{i}$, where $v_{i}\in\mathbb{R}$ and $e_{i}$ the standard $i$th basis vector,
\begin{equation*}
    \tilde{\omega}_0(v)\leq\tilde{\omega}_0(0).
\end{equation*}
Hence $0$ is a global maximum of $\tilde{\omega}_0$ along the $i$th axis. If $\partial_{i}\tilde{\omega}_0(0)$ exists, this would mean
\begin{equation} \label{partialzero}
    \partial_{i}\tilde{\omega}_0(0)=0,
\end{equation}
and if all the partial derivatives $\partial_{i}\tilde{\omega}_0$ exist in a vicinity of $0$ and are continuous at $0$, it would entail that $\nabla\tilde{\omega}_0(0)=0$, which is \eqref{omegader}. To show these facts, we use a differentiation technique in fluid mechanics. It goes as follows.\\
Let $\xi(\tau)$ be a smooth vector field, $\tau\in\mathbb{R}$, and let $\xi'(\tau)=:\upsilon(\tau)$. Let $\mathcal{F}(\xi(\tau),\tau)$ be a density function of space and time. Let $R(\tau)$ be a region varying with time. Then the Reynold's transport equation \cite{gatski2013compressibility} states that
\begin{equation} \label{reygatski}
    \frac{d}{d\tau}\bigg(\int_{R(\tau)} \mathcal{F}(\xi(\tau),\tau)\,d\xi(\tau)\bigg) = \sum_{l}\int_{R(\tau)} \partial_{l}\{\mathcal{F}(\xi(\tau),\tau)(\upsilon(\tau))_{l}\} \,d\xi(\tau) + \int_{R(\tau)} \partial_{\tau}\{\mathcal{F}(\xi(\tau),\tau)\}\,d\xi(\tau),
\end{equation}
where $\partial_{\tau}$ denotes the partial derivative in terms of the time argument $\tau$. 

\begin{figure}
    \centering
    \includegraphics[width=1.1\textwidth]{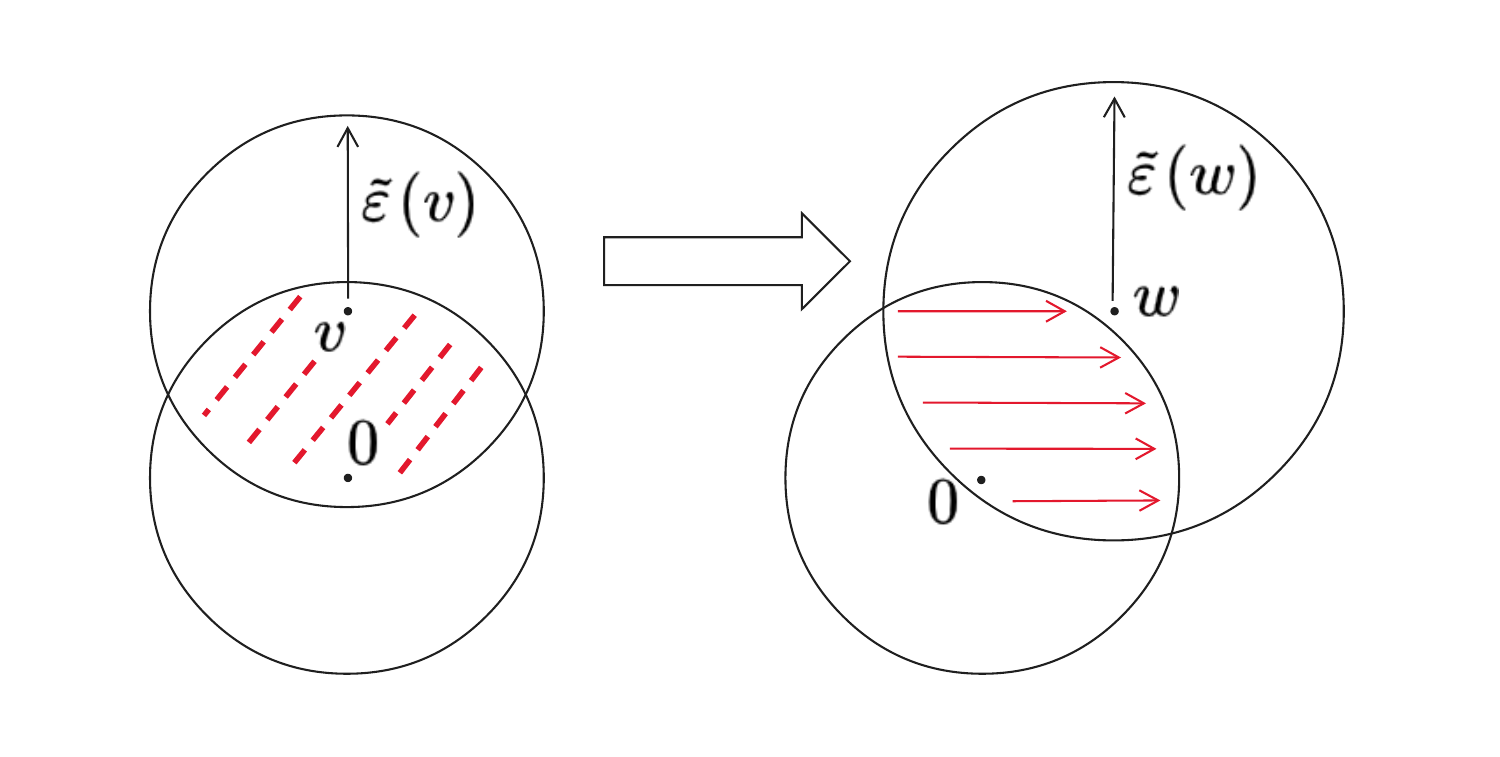}
    \caption{Interpreting $\tilde{\omega}_0$ as a moving mass. (Left): The mass at time $\tau=0$. (Right): The mass at some time $\tau^{*}$, and $v$ moves to $w=v+e_{i}\tau^{*}$.}
    \label{fig:reyn}
\end{figure}

We apply \eqref{reygatski} to our context. Take $v\in B(0,2\varepsilon^{\star})$. Let $\gamma_{v}(\tau):= v + e_{i}\tau$ and $R_{v}(\tau):= B(0,\varepsilon)\cap B(\gamma_{v}(\tau),\tilde{\varepsilon}(\gamma_{v}(\tau)))$, for $\tau\in\mathbb{R}$. Let $\xi^{i}(\tau)$ denote a constant-velocity vector flow with $(\xi^{i})'(\tau)=e_{i}$. Then (see Figure \ref{fig:reyn})
\begin{equation*}
    \partial_{i}\tilde{\omega}_0(v) = (\alpha n/k)\frac{d}{d\tau}\bigg(\int_{R_{v}(\tau)}\tilde{p}(\xi^{i}(\tau))J_{x}(\xi^{i}(\tau))\,d\xi^{i}(\tau)\bigg)|_{\tau=0},
\end{equation*}
if exists. Comparing this to \eqref{reygatski}, we obtain
\begin{equation} \label{reyn}
    \frac{d}{d\tau}\bigg(\int_{R_{v}(\tau)}\tilde{p}(\xi^{i}(\tau))J_{x}(\xi^{i}(\tau))\,d\xi^{i}(\tau)\bigg) = \int_{R_{v}(\tau)} \partial_{i}\{\tilde{p}(\xi^{i}(\tau))J_{x}(\xi^{i}(\tau))\} \,d\xi^{i}(\tau).
\end{equation}
Since the exponential map $\exp$ (\eqref{expmap}) is guaranteed to be a diffeomorphism by Assumption \eqref{assumption} and $p\in C^2(\mathcal{M})$, the derivative in \eqref{reyn} exists, and so does $\partial_{i}\tilde{\omega}_0(v)$. Hence \eqref{partialzero} holds. It remains to observe that $\gamma_{v}$ is continuous with respect to $v$, and so is $\tilde{\varepsilon}(\gamma_{v})$. Therefore, $Vol_{m}(R_{v}(0))\to Vol_{m}(R_{0}(0))$ in \eqref{reyn}; this is enough to conclude that $\partial_{i}\tilde{\omega}_0(v)\to \partial_{i}\tilde{\omega}_0(0)$ when $v\to 0$, for all $i$, and consequently, that $\nabla\tilde{\omega}_0(0)=0$.\\

Next, we consider $\partial_{j}\partial_{i}\tilde{\omega}_0(v)$. To do so, we define $\beta_{v}(t,\tau):= v + e_{j}t + e_{i}\tau$, for $t,\tau\in\mathbb{R}$ and
\begin{equation*}
    R_{v}(t,\tau) := B(0,\varepsilon)\cap B(\beta_{v}(t,\tau),\tilde{\varepsilon}(\beta_{v}(t,\tau))).
\end{equation*}
Let $\xi^{ij}(t,\tau)$ be such that $\partial_{t}\xi^{ij}(t,\tau) =e_{j}$ and $\partial_{\tau}\xi^{ij}(t,\tau) =e_{i}$. By replacing $v$ with $v+e_{j}t$ in \eqref{reyn}, we can write, 
\begin{equation*} 
    \partial_{j}\partial_{i}\tilde{\omega}_0(v)
    = (\alpha n/k) \frac{d}{dt}\bigg(\int_{R_{v}(t,0)} \partial_{i}\{\tilde{p}(\xi^{ij}(t,0))J_{x}(\xi^{ij}(t,0))\} \,d\xi^{ij}(t,0)\bigg)|_{t=0}.
\end{equation*}
Invoking \eqref{reygatski} again, we have
\begin{equation} \label{second1}
    \partial_{j}\partial_{i}\tilde{\omega}_0(\beta_{v}(t,0)) = (\alpha n/k)\int_{R_{v}(t,0)} \partial_{j}\partial_{i}\{\tilde{p}(\xi^{ij}(t,0))J_{x}(\xi^{ij}(t,0))\} \,d\xi^{ij}(t,0)=:(\alpha n/k)I(t).
\end{equation}
Due to the discussed regularity of $p$ and the exponential map, $I(t)$ exists. Furthermore, let $t=0$ in \eqref{second1}. Then since $R_{v}(0,0)=B(0,\varepsilon)\cap B(v,\tilde{\varepsilon}(v))$,
\begin{equation} \label{second2}
    (\alpha n/k) |I(0)|\leq C_{p}(\alpha n/k)\int_{R_{v}(0,0)}\,du \leq C_{p}(\alpha n/k)Vol_{m}(B(0,\varepsilon))\leq C_{p}
\end{equation}
where we've denoted $u=\xi^{ij}(0,0)$. The continuity of $\partial_{j}\partial_{i}\tilde{\omega}_0(v)$ with respect to $v$ is again easily observed; we conclude from \eqref{second1}, \eqref{second2} that $\tilde{\omega}_0\in C^2(B(0,2\varepsilon^{\star}))$ and $\|\tilde{\omega}_0\|_{C^2(B(0,2\varepsilon^{\star}))}=O(1)$.


\end{document}